\documentclass{article}

\usepackage[preprint]{neurips_2026}

\usepackage[utf8]{inputenc} % allow utf-8 input
\usepackage[T1]{fontenc}    % use 8-bit T1 fonts
\usepackage{hyperref}       % hyperlinks
\usepackage{url}            % simple URL typesetting
\usepackage{booktabs}       % professional-quality tables
\usepackage{amsfonts}       % blackboard math symbols
\usepackage{nicefrac}       % compact symbols for 1/2, etc.
\usepackage{microtype}      % microtypography
\usepackage{xcolor}         % colors

\usepackage{graphicx}
\usepackage{tikz}
\usepackage{algorithm}
\usepackage{algpseudocode}
\usepackage{xspace}
\usepackage{mathtools}
\usepackage[super]{nth}
\usepackage{amsmath}
\usepackage{amsthm}
\usepackage{amsfonts}
\usepackage{amssymb}
\usepackage[disable]{todonotes}

\usepackage{subcaption}
\usepackage{bm}

\theoremstyle{plain}
\newtheorem{theorem}{Theorem}[section]
\newtheorem{definition}[theorem]{Definition}
\newtheorem{lemma}[theorem]{Lemma}
\newtheorem{corollary}[theorem]{Corollary}

\newcommand{\smallsep}{\newline\rule{0em}{\baselineskip}}
\newcommand{\safeMEB}{\ensuremath{\mathrm{safeMEB}}\xspace}
\newcommand{\MEB}{\ensuremath{\mathrm{MEB}}\xspace}
\newcommand{\TB}{\ensuremath{\mathrm{TB}}\xspace}
\newcommand{\convexhull}{\ensuremath{\mathrm{Conv}}\xspace}

\title{Practical Validity Conditions for \\Byzantine-Tolerant Federated Learning}

\author{%
  M\'elanie Cambus \\
  Aalto University, Finland\\
  \texttt{melanie.cambus@aalto.fi} \\
  \And
  Darya Melnyk\thanks{Supported by German Research Foundation (DFG), SPP 2378, project ReNO-2 (511099228), 2025-2029} \\
  Technical University Berlin \\
  \texttt{melnyk@tu-berlin.de} \\
  \AND
  Tijana Milentijević\footnotemark[1] \\
  Technical University Berlin \\
  \texttt{tijana.milentijevic@tu-berlin.de} \\
  \And
  Stefan Schmid\footnotemark[1] \\
  Technical University Berlin \\
  \texttt{stefan.schmid@tu-berlin.de} \\
}

\begin{document}

\maketitle

\begin{abstract}
  Robust aggregation is the core operation in Byzantine-tolerant federated learning. To ensure the quality of aggregation independently of data distribution or attacks, validity conditions are needed. They provide geometric guarantees of where the output of the aggregation must lie. The widespread convex validity requires the output to lie in the convex hull of the honest vectors. Although this guarantee is strong in theory, it is poorly suited to modern federated learning systems, as it has dimension-dependent resilience and excludes many practical aggregation rules.
We introduce the minimum enclosing ball (\MEB) validity condition for robust aggregation, as well as its multiplicative relaxation, $c$-\MEB validity, where $c$ is a constant. We show that exact \MEB validity still suffers from limited resilience, while relaxed $c$-\MEB validity is achievable if a majority of clients is honest, i.e. $n>2t$. We give an optimal MinMax-\MEB rule for the relaxed condition with the bound $c<\sqrt{2}$ and prove explicit relaxed-\MEB guarantees for standard aggregators including minimum-diameter averaging, medoid and geometric median. Finally, we relate \MEB validity to convex, relaxed-convex and box validity studied in prior literature, thus providing a systematic map of geometric validity conditions for Byzantine-robust aggregation. 
Our results show that relaxed \MEB validity connects validity conditions in distributed computing and Byzantine-tolerant aggregation rules, and offers a practical alternative to convex validity.
%\MEB validity bridges the gap between theoretical guarantees and practical robust aggregation by allowing controlled deviations from the convex hull, while remaining close to honest nodes. 
\end{abstract}

\section{Introduction}

%Vector aggregation is an important subroutine in distributed machine learning (ML), where several nodes collaborate to build an ML model without sharing the training data. Distributed ML is necessary when, e.g., the training data cannot be shared among the nodes for privacy reasons, or when the amount of training data is too large to be stored by a single node. In the literature, several distributed ML approaches have been suggested. One alternative, which is particularly well-suited when the stochastic gradient descent method is used to optimize the model, is to let the nodes agree on a common gradient before they update their local copy of the model parameters. Another option is to let the nodes build a common view by directly exchanging the model parameters. In either case, the update of local gradients or parameters can be seen as an execution of vector consensus. The quality of the final ML model can suffer if one or more nodes use a poisoned dataset, if one or more nodes fail to compute the gradients correctly, or if a failure in the communication network corrupts some of the exchanged data. To address these issues, distributed ML algorithms that can tolerate Byzantine nodes, i.e., nodes that deviate arbitrarily from the protocol, have been proposed in the literature.

Federated Learning (FL) systems aggregate high-dimensional updates computed by many clients~\cite{mcmahan2017communication}. 
When some clients are \textit{Byzantine}~\cite{lamport2019byzantine} and exhibit unplanned or even malicious behavior, the server can no longer trust the received vectors, since malicious clients may send arbitrary updates~\cite{yin2018byzantine, kairouz2021advances, 10.1145/3616537, 10.5555/3780338.3780384, allouah2025adaptive}. Also benign failures or poisoned data can influence aggregation. As distributed learning is increasingly deployed in critical domains such as healthcare~\cite{nguyen2022federated} or banking~\cite{long2020federated}, robustness to Byzantine clients has become a central concern. 
Robust aggregation rules are therefore used to ensure that the vector applied by the server remains representative of the honest (i.e., non-Byzantine) clients. A central question is how to define this requirement independently of data distribution, or attacks.

This can be formalized using \emph{validity conditions}, which specify where the aggregation output is allowed to lie as a function of the honest input vectors. Validity conditions are hence a foundation for Byzantine-tolerant FL. 
%Before analyzing the convergence of a learning algorithm, one needs to understand whether the aggregation step can output a vector that is geometrically consistent with the honest updates. 
When all nodes have identically distributed data, nodes that show Byzantine behavior can be viewed as outliers~\cite{yin2018byzantine,NIPS2017_f4b9ec30,  guerraoui2018medoid}. This makes it possible to filter out such nodes in the training and achieve agreement close to the desired output (typically the average of non-faulty gradients). However, if the data is distributed heterogeneously, the quality of the ML model may drop if outliers are removed from the training process. 
Therefore, researchers have recently made great efforts to develop aggregation methods for heterogeneous data, including trimmed mean~\cite{pmlr-v80-yin18a, 10.5555/3540261.3542179}, minimum diameter averaging (MDA)~\cite{10.5555/3540261.3542179}, geometric median~\cite{9721118, spaa-geom-median}, medoid~\cite{guerraoui2018medoid, bhowmick2022resilient, xie2018generalized}, Krum~\cite{NIPS2017_f4b9ec30,  xie2018generalized,wang2025federated}, and coordinate-wise median~\cite{pmlr-v80-yin18a, asia-ccs-coord-median}. These aggregation rules restrict the output region in different ways, and hence provide different notions of aggregation quality.

%Aggregation of vectors under Byzantine adversaries has also been considered from the perspective of distributed computing. For a long time, however, the focus in distributed computing was to find algorithms for agreeing on scalars, and multi-dimensional agreement was viewed as a generalization of multi-valued agreement executed in every dimension. 

The classical validity condition for multidimensional vector aggregation is convex validity~\cite{VectorConsensusAsynch, VectorConsensus, mendes2015multidimensional}, which requires the output to lie in the convex hull of the honest clients' inputs. Although this validity condition provides a strong geometric guarantee, it is not suitable for modern FL systems: In $d$ dimensions, it can tolerate at most $t < n/(d+1)$ Byzantine clients, which is not suitable for high-dimensional models.
% Additionally, as shown by the seminal works of \citet{VectorConsensusAsynch}, \citet{VectorConsensus}, and \citet{mendes2015multidimensional}, convex validity condition cannot be satisfied if algorithms are executed in each dimension independently. 
%This validity condition requires the output of an agreement algorithm to be inside the convex hull of all honest nodes. 
%Soon after these seminal papers were published, convex validity has been extended to various network models~\cite{fugger_et_al:LIPIcs.DISC.2018.27,ghinea2025convex} and was successfully applied for approximate agreement on graphs~\cite{nowak_et_al:LIPIcs.DISC.2019.29}. 
%In practical applications, researchers often apply methods from robust statistics. They compute the geometric median, medoid, minimum diameter... Methods such as the median have been shown to have a breaking point of $1/2$, meaning that they can withstand half of corrupted samples. \todo[inline]{really need to read up on robust statistics and see how Rachid and others motivate the aggregation rules that they use}
%While convex validity provides strong geometric guarantees, it also imposes a strict requirement on the number of Byzantine nodes $t$, namely that: $t<n/(d+1)$. This is is too restrictive in practical applications, as modern systems have high-dimensional inputs.
%However, for practical applications that apply vector consensus in $\mathbb{R}^d$, convex validity is not suitable due to the low number of Byzantine failures that it can tolerate ($t<n/(d+1)$). 
Therefore, practical algorithms for Byzantine-tolerant vector aggregation violate convex validity by explicitly allowing the output to leave the convex hull of all honest vectors. Note that the non-faulty vectors may lie in a lower-dimensional space, while the Byzantine vectors are placed in the full-dimensional space.  %under data heterogeneity, honest updates may lie close to a low-dimensional structure, while Byzantine vectors may introduce additional directions. 
Practical aggregators may still output useful vectors outside the convex hull of the honest updates, as long as the output remains close to the honest population. Thus, convex validity is often too strict to explain the behavior of robust aggregation in federated learning.
%In practice, these methods have shown to be effective for applications such as Federated Learning \cite{8950073}, distributed optimization \cite{su2016fault} and robust sensor fusion in cyber-physical systems \cite{8669811}. 

This work develops a systematic validity framework for Byzantine-tolerant vector aggregation. We introduce a novel validity condition --- the \emph{minimum enclosing ball} (\MEB) validity, which requires output to lie in the smallest enclosing ball of the honest vectors. Unlike convex validity, \MEB validity defines a fully-dimensional ball of all non-faulty vectors and naturally captures aggregation outputs that leave the convex hull, while remaining within the geometric scale of the honest updates. However, we show that \MEB validity is still too strict, as it also suffers from resilience limitations that depend on the geometry of the input. This motivates a relaxed version, called $c$-\MEB validity, where the radius of the honest minimum enclosing ball is increased by a multiplicative factor $c$.
Our results show that relaxed \MEB validity provides a useful bridge between validity conditions in distributed computing and Byzantine-tolerant aggregation rules. Overall, \MEB validity offers a practical alternative to convex validity, ensuring that aggregation remains close to the honest clients.
%, while accommodating the behavior of modern robust aggregation rules. 

%We want to emphasize that this work is investigating validity conditions on a fundamental level, while putting the specific agreement protocols to which these conditions are applied in the background. The results in this paper can therefore be seen as if they were presented in a client-server model where a central server aggregates vectors and decides on an output based on a given validity condition. Observe that the client-server model can be simulated using authenticated interactive consistency for agreement protocols. Transferring the results to the original setting of multidimensional approximate agreement, where the convex validity has first been presented, is possible, but it is beyond the scope of this work.

%\subsection{Our contribution}

%In this paper, we introduce a practical, robust and efficient-to-compute validity conditions: The \MEB validity allows the nodes to agree inside a minimum enclosing ball of all non-faulty vectors; Its relaxation, the  $c$-relaxed \MEB validity, allows agreement inside the minimum enclosing ball, whose radius has been increased by a factor of $c$. 

\noindent\textbf{Our contributions:}

\begin{enumerate}
    \item \textbf{A validity framework for robust aggregation.}
    We introduce \MEB validity and its multiplicative relaxation, $c$-\MEB validity, where $c$ is a constant. These conditions provide full-dimensional geometric guarantees on the aggregation output and are motivated by the behavior of practical Byzantine-tolerant aggregation rules.

    \item \textbf{Limits of exact \MEB validity.} We show that exact \MEB validity requires agreement inside the intersection of all candidate minimum enclosing balls, analogously to the safe area for convex validity~\cite{VectorConsensus}. We prove a lower bound showing that exact \MEB validity still suffers from poor resilience in high-dimensional settings.
    
    \item \textbf{Optimal solution for $\bm{c}$-relaxed \MEB validity.} 
    We present MinMax-\MEB, an optimal aggregation rule for $c$-\MEB validity, that is based on solving a minimization problem in Section~\ref{sec:optimal} under $n>2t$. We also present an upper bound for this algorithm, showing that the relaxation factor can be upper bounded by $c<\sqrt{2}$. For $d=2$, this factor is tight.
    
    \item \textbf{Validity guarantees for practical aggregators.} In Section \ref{sec:existing_algorithms}, we show that existing algorithms, including MDA, medoid, and geometric median, satisfy $c$-relaxed \MEB validity, where $c$ is at most a constant. Interestingly, the presented relaxation factor for these algorithms depends on the ratio between $t$ and $n$, and it is smaller, the smaller the ratio. 

    %\item \textbf{Impossibility results and lower bounds for \MEB validity.} 
    %We show that agreement inside the \MEB requires the nodes to agree inside the intersection of all possible \MEB. This is a similar condition to agreeing inside the safe area for convex validity~\cite{VectorConsensus}. In Section~\ref{sec:lower_bounds} we present a lower bound for \MEB validity, showing that this validity condition also suffers from low resilience of up to $t<n/(d+1)$ Byzantine nodes. Observe that the lower bound for convex validity cannot be extended to \MEB validity, as the minimum enclosing balls in the lower bound of~\cite{VectorConsensus} would have a common intersection. We therefore present a new construction that relies on intersections of subsets of minimum enclosing balls, and relies on the fact that some of the balls in this construction must have a considerably larger radius than others.
    
    \item \textbf{A systematic comparison of validity notions.} 
    We compare \MEB validity to existing validity conditions, such as convex validity, $(\delta,p)$-relaxed convex validity~\cite{xiang_et_al:LIPIcs.OPODIS.2016.26}, and box validity~\cite{10.1007/978-3-032-11127-2_10} in Section~\ref{sec:relations}. This clarifies which validity conditions imply others and puts \MEB validity into perspective with existing vector aggregation conditions.
\end{enumerate}

\section{Background and Related Work}
%\todo[inline]{todo, make more FL oriented}
% \noindent\textbf{Applications of Vector Consensus\footnote{Here we use vector consensus to describe the process where nodes have input vectors in multiple dimensions and need to agree on a common output vector. Note that there is a different notion of vector consensus compared to~\cite{10.1145/277697.277772,1524949}, where the entries of the vector represent single-valued outputs for each node.}.} In the areas of distributed optimization, federated and collaborative learning, agreement algorithms are used as a subroutine when different gradients or model parameters need to be aggregated in the optimization or training process. 
\noindent\textbf{Aggregation in federated learning.}
The quality of an aggregation vector in federated learning is usually measured in terms of absolute distances from e.g. the average~\cite{10.5555/3540261.3542179}, in terms of resilience~\cite{ NIPS2017_f4b9ec30, 10.1145/3465084.3467902}, or robustness~\cite{pmlr-v206-allouah23a}. In~\cite{10.1145/3465084.3467902} the $(f,\varepsilon)$-Byzantine resilience is defined which requires the output to be at a distance of at most $\varepsilon$ from the average of honest nodes. In~\cite{NIPS2017_f4b9ec30}, $(\alpha,f)$-Byzantine resilience is tied to the moments of the output function, which are controlled by the moments of the honest gradient estimator. In~\cite{pmlr-v206-allouah23a}, $(f,\kappa)$-robustness is defined based on the averages of every subset of $n-t$ nodes, which also includes the Byzantine nodes.
Proposed algorithms for the above measures apply relatively simple filtering options of the inputs, such as GradFilter~\cite{10.1145/3465084.3467902} or the Reliable Broadcast --- Trimmed Mean~\cite{10.5555/3540261.3542179}, they use robust aggregation methods, such as medoid~\cite{guerraoui2018medoid}, coordinate-wise median~\cite{pmlr-v80-yin18a, asia-ccs-coord-median}, minimum diameter averaging (MDA)~\cite{10.5555/3540261.3542179}, or geometric median~\cite{9721118, spaa-geom-median}, and they apply combinations of these rules, see for example Krum~\cite{NIPS2017_f4b9ec30}, Multi-Krum~\cite{guerraoui2018medoid}, soft medoid~\cite{geisler2020reliable} and trimmed soft medoid~\cite{11126666}.

%The medoid idea is used in the context of federated learning (may need to find some more original paper) \cite{10.1145/3616537, NIPS2017_f4b9ec30}

\noindent\textbf{Validity conditions for agreement protocols.}
Validity conditions in agreement protocols ensure that the output space is restricted based on the inputs of the nodes. This makes validity conditions similar to the quality measures in Byzantine-tolerant vector aggregation discussed above. %While the measures are similar to validity conditions, the aggregation methods vastly differ from the algorithms considered for agreement protocols in distributed computing.
Traditional validity conditions have been defined for binary input values~\cite{Bracha87,10.1145/800221.806707,10.1145/322186.322188,10.1145/301308.301368, 9546489, 10.1145/3293611.3331591}. 
% In one-dimensional binary agreement, strong validity is the most common validity condition applied in the literature~\cite{Bracha87,10.1145/800221.806707,10.1145/322186.322188}. This validity condition, sometimes also referred to as unanimity, requires the output to be one of the honest nodes' inputs. 
% A weaker version of this condition, the weak validity condition, was considered for applications in $k$-set agreement~\cite{10.1145/301308.301368} and BFT protocols~\cite{9546489, 10.1145/3293611.3331591}. This validity condition only applies if the honest nodes are unanimous, and only if there are no faults in the system. 
In the more practical multi-valued setting (e.g., with input values in $\mathbb{R}$), the correct-proposal validity condition~\cite{10.1145/872035.872066,SIU1998157} has been considered. This condition requires the output to be one of the honest inputs. 
The convex validity condition was first introduced in the one-dimensional setting to solve approximate agreement~\cite{10.1145/5925.5931}. This validity condition is typically also applied in the multi-valued setting, and it requires the output to be in the range of the honest values.
Other stronger validity conditions than convex validity have appeared in the literature, see for example median validity~\cite{MedianValidity} or interval validity~\cite{intervalValidity}.
A comprehensive comparison of validity conditions for one-dimensional agreement is provided in~\cite{10.1145/3583668.3594567}.

Observe that one-dimensional validity conditions can be applied coordinate-wise to satisfy validity in multiple dimensions~\cite{xiang_et_al:LIPIcs.OPODIS.2016.26,  10.1007/978-3-032-11127-2_10, intervalValidity}. The corresponding validity condition is referred to as box validity in this work.
Convex validity in multiple dimensions has first been mentioned in \cite{VectorConsensusAsynch} and \cite{VectorConsensus}, as well as in the journal version~\cite{mendes2015multidimensional}. In this set of work, the output has to be inside the convex hull of the honest input vectors. Using Helly's theorem, the authors show that, to satisfy convex validity, it is sufficient to find a vector in the intersection of the convex hulls of all subsets of $n-t$ vectors. They also provide lower bounds showing that $t<n/(d+1)$ is a necessary condition for convex validity. 
To make the validity condition more practical, relaxations of convex validity have been considered~\cite{xiang_et_al:LIPIcs.OPODIS.2016.26}. Let $CH$ denote the convex hull of all honest vectors. In the $(\delta,p)$-relaxed vector agreement, the output of an algorithm can be inside $\{u\,\big\vert\,\Vert u - v \Vert_p \le \delta\ \ \forall\ v\in CH \}$~\cite{xiang_et_al:LIPIcs.OPODIS.2016.26}.

% Convex validity also appears in generalizations of approximate agreement to graphs and, more generally, convexity spaces~\cite{nowak_et_al:LIPIcs.DISC.2019.29}.
% An overview of applications of convex validity in other network models can be found in~\cite{ghinea2025convex}. Other similar applications appear in 
% lattice agreement in the shared memory~\cite{10.1007/BF02242714,10.1137/S0097539700366000} and the message passing models~\cite{10.1145/2332432.2332458, zheng_et_al:LIPIcs.DISC.2018.41}

In contrast to previous work, the \MEB validity conditions from this paper are directly motivated by the practical algorithms, such as MDA~\cite{10.1007/978-3-032-11127-2_10}, that allow the output to be outside of the convex hull of honest vectors, but still bound the distance by the maximum diameter of the honest vectors.

\section{Model}\label{sec:model}
We consider the distributed client-server model where the server aggregates input vectors $v\in\mathbb{R}^d$ (gradients or model parameters) of its $n$ nodes (clients) over reliable communication links. The goal of the server is to decide on an agreement vector that best represents the inputs of the non-faulty nodes. 

We assume that up to $t<n/2$ of the nodes can be faulty and show arbitrary behavior. In particular, we assume that these nodes are Byzantine: They know the input vectors of all other nodes, as well as the agreement algorithm, and they are allowed to collaborate. However, Byzantine nodes are node faults and thus cannot occupy other nodes in the network. We use $H$ to denote the set of honest, i.e., non-faulty, nodes. Note that $|H|\ge n-t$.

We can view the agreement problem as an aggregation method that takes the inputs of the nodes and outputs a single vector. For the Byzantine nodes, we assume that their collected vectors do not have to correspond to their input vectors. To estimate the quality of the aggregation vector, we will consider distances between vectors. 
%To make sure that agreement is not established on an arbitrary vector, we use validity conditions.  
In this work, we will measure the distance between vectors in terms of the Euclidean distance, which is defined as follows:
\begin{definition}[Euclidean Distance]
For any $v, v' \in \mathbb{R}^d$, the Euclidean distance between $v$ and $v'$ is
    $\|v-v'\|_2 = \sqrt{ \sum_{i = 1}^d (v_i - v'_i)^2}$, where  $v_i$ is the projection of $v$ on coordinate $i\in [d]$.
\end{definition}

\noindent\textbf{Vector aggregation with \MEB validity.}
In the following, we formally define the \MEB and $c$-relaxed \MEB validity conditions. 
%In Section~\ref{sec:lower_bounds}, we present the lower bounds for the validity conditions. Section~\ref{sec:optimal} presents an optimal algorithm to solve vector aggregation with $c$-relaxed \MEB validity and presents an upper bound on $c$. Finally, in Section~\ref{sec:existing_algorithms}, we discuss the relaxation factors from the practical algorithms in the literature. Finally, in Section~\ref{sec:relations}, we compare the introduced \MEB validity conditions to the existing validity conditions for vector aggregation.

%\subsection{Definition of \MEB validity}\label{sec:def_MEB}

%Now, we define the minimum enclosing ball and show that it has useful properties.

\begin{definition}[Minimum Enclosing Ball (\MEB)~\cite{smallest-enclosing-ball}]
    Given a set of vectors $S$ in $\mathbb{R}^d$, the minimum enclosing ball $\MEB(S)$ is the smallest ball in $\mathbb{R}^d$ containing all vectors in $S$. 
\end{definition}

In $\mathbb{R}^d$, the \MEB can be defined with up to $d+1$ vectors. Observe, however, that only two vectors may be sufficient to define the ball. It has been shown in the literature that \MEB can be computed efficiently, using, for example, Welzl's algorithm~\cite{welzlsMEBalgo}. An important property of the minimum enclosing ball is that it is unique. This allows us to use this ball to define a new validity condition:

\begin{definition}[\MEB validity]
    \MEB validity condition requires that the output vector of each non-faulty node must lie inside the minimum enclosing ball of the input vectors of all non-faulty nodes denoted by \MEB.
\end{definition}
Since the \MEB is unique, the presented \MEB validity condition is well defined. This would not be the case for the diameter of honest vectors (the maximum distance between any two honest vectors), which is a widespread baseline used in practical applications. However, the radius of \MEB can be used to lower bound the diameter of honest vectors, and the diameter of \MEB can be used as an upper bound. 

The exact \MEB validity is still too strong for high-dimensional applications: as we show in Section~\ref{sec:lower_bounds}, it suffers from resilience limitations that depend on geometry of the input space, similarly to convex validity. This motivates a relaxed version of \MEB validity. The goal of the relaxation is to exploit the honest majority threshold $n>2t$, in which practical Byzantine-tolerant aggregation rules perform. 

\begin{definition}[$c$-\MEB validity]
    Let $H$ be the set of input vectors of honest nodes, and let \MEB denote their minimum enclosing ball with center in $C_{\MEB}^*$ and radius $r^*$. An algorithm satisfies $c$-\MEB validity if every non-faulty node outputs a vector $y$ such that $\|y - C_{\MEB}^*\|_2 \le c\cdot r^*$.
\end{definition}
For $c=1$, this condition reduces to exact \MEB validity. We will refer to $c$ as the relaxation factor in this work, and sometimes call the validity condition the $c$-relaxed \MEB validity condition.

\paragraph{Implication for honest-gradient bias.}
A useful consequence of $c$-\MEB validity is that it directly bounds the deviation of the aggregation rule from the average honest update. Let $\bar h = \frac{1}{|H|}\sum_{h\in H} h$ denote the average of the honest client updates and let $C^*_{\mathrm{MEB}}$ and $r^*$ be the center and radius of the honest minimum enclosing ball. Since $\bar h$ lies in the convex hull of the honest updates and the convex hull is contained in the \MEB, we have $\|\bar h - C^*_{\mathrm{MEB}}\|_2 \le r^*$.
Therefore, any output $y$ satisfying $c$-\MEB validity obeys $\|y-\bar h\|_2\le\|y-C^*_{\mathrm{MEB}}\|_2+\|C^*_{\mathrm{MEB}}-\bar h\|_2\le(c+1)r^*$.
Thus, $c$-\MEB validity can be interpreted as a distribution-independent bound on the bias of the aggregate relative to the honest average. This connects the geometric validity condition to standard analyses of Byzantine-tolerant federated optimization, where the effect of aggregation is often controlled through the deviation from the honest average update.

In order to satisfy \MEB validity conditions, we will show that it is necessary to agree inside the intersection of \MEB of all possible $n-t$ subsets of vectors. This area is defined analogously to the safe area for convex validity~\cite{VectorConsensus}:

\begin{definition}[\safeMEB]
    Let $S$ be a set of vectors in $\mathbb{R}^d$, with $|S|\ge n-t$. Then, the safe area for \MEB validity, denoted \safeMEB, is defined as:$$\safeMEB_t = \bigcap\limits_{T\subseteq S, |T|=n-t} \MEB(T).$$
    For $c$-\MEB validity, we define the $c$-\safeMEB area analogously, by replacing the intersection of minimum enclosing balls with the intersection of balls with centers in $C^*_{\MEB(T)}$ and radius $c\cdot r^*$.
\end{definition}

%\section{Possibility and impossibility for \MEB validity}

\section{Necessary conditions for \MEB validity}\label{sec:lower_bounds}

In this section, we first present the necessary conditions that are needed to satisfy \MEB and $c$-\MEB validity, respectively. We then present an impossibility result showing that \MEB validity cannot be satisfied if $n\le(d+1)t-d$.

\begin{theorem}\label{thm:safeAgreementNecessary}
    \MEB validity requires agreement inside $\safeMEB_t$.
\end{theorem}
\begin{proof}
    Consider an agreement algorithm execution where all $t$ nodes are faulty and assume that the faulty nodes follow the protocol with their own input vectors chosen in a worst-case manner. Observe that such an execution is indifferentiable from an execution where all nodes are honest. In the following, we will prove the theorem statement by contradiction. Assume that the agreement vector lies outside of some minimum enclosing ball $\MEB_i$. Consider the $n-t$ vectors for which this ball is the \MEB. If these input vectors are all from honest nodes, then the algorithm trivially violates the \MEB validity condition. Assume next that there are faulty vectors among the $n-t$ vectors that span $\MEB_i$. We can construct an identical execution of the algorithm where we swap the roles of the honest and the faulty nodes, such that $\MEB_i$ is defined with honest nodes only. Observe that, since the faulty nodes correctly followed the protocol, such a swap is always possible. Since both considered executions are identical to the algorithm, the algorithm has to output a vector outside on $\MEB_i$ in both cases. This would lead to a contradiction also in the second case. 
\end{proof}

The previous proof has not relied on the position of the \MEB, but only on their intersection properties. Therefore, the above proof also works for $c$-\MEB validity:

\begin{corollary}
    $c$-\MEB validity for $c\ge 1$ requires agreement in the $c$-$\safeMEB_t$.
\end{corollary}

We next focus on the resilience that can be achieved under \MEB validity.
Observe that the lower bound of $n>(d+1)t$ for convex validity does not directly translate to the lower bound on \MEB validity, as all minimum enclosing balls of the convex hulls in~\cite{VectorConsensus} do have a common intersection. In the following, we present an almost tight resilience lower bound for \MEB validity, showing that at least $n>(d+1)t-d$ nodes are needed to satisfy \MEB validity.

\begin{theorem}\label{thm:BallValidityLB}
    No agreement algorithm can satisfy \MEB validity with $n\le(d+1)t-d$ nodes, where $t\ge d-1$.
\end{theorem}
\begin{proof}
    Assume by means of contradiction that for $n=(d+1)t-d$, the (minimum enclosing) balls of $dt-d$ nodes must all intersect in at least one point. We will show that it is possible to construct a counterexample where a subset of $d+1$ balls does not have a common intersection. 

    We start the construction with $d$ balls of equal size, where ball $i\in [d]$ has its center at the coordinate $\sqrt{\frac{d}{d-1}}\cdot u_i$, where $u_i$ is the $i$-th unit vector. We set the radius of each ball to $1$. These balls are chosen such that they intersect in exactly one point. It can be easily verified that this point has the coordinates $\frac{1}{\sqrt{d(d-1)}}\cdot(1,\ldots,1)$, and that it is the only intersection point, as the intersection point of all balls must be of the form $x\cdot(1,\ldots,1)$. Any subset of $d-1$ of these balls has an intersection point in the coordinate $\frac{\sqrt{d}+1}{(d-1)^\frac{3}{2}}\cdot \bar{u}_i$, where $\bar{u}_i$ is constructed from $u_i$ by flipping all bits. %\todo{need to rename coordinates. Here only one coordinate is $0$}

    Consider now the hyperplane defined by the $d$ points $\frac{\sqrt{d}+1}{(d-1)^\frac{3}{2}}\cdot \bar{u}_i$. This hyperplane lies at a distance of $\frac{1}{d(d-1)}>\frac{1}{d}$ from the intersection of all $d$ balls in $\frac{1}{\sqrt{d(d-1)}}\cdot(1,\ldots,1)$. We can approximate this hyperplane using the $(d+1)$-st ball from the construction. Let the midpoint of this ball lie in coordinate $y\cdot(1,\ldots,1)$, where $y$ is an arbitrarily large constant. We can choose a radius for this ball such that the ball touches the intersection of every subset of $d-1$ balls, similar to the hyperplane. 

    We now explain how the input vectors are distributed such that the defined $d+1$ balls correspond to minimum enclosing balls in the construction. We place $t-1$ input vectors in $\frac{1}{\sqrt{d(d-1)}}\cdot(1,\ldots,1)$ and in each of the coordinates $\frac{\sqrt{d}+1}{(d-1)^\frac{3}{2}}\cdot \bar{u}_i$. These are $(d+1)(t-2)= (d+1)t - 2d - 2$ input vectors. For each of the balls of radius $1$, we place one input vector in the coordinates $2\sqrt{\frac{d}{d-1}}u_i + \frac{1}{\sqrt{d(d-1)}}\bar{u}_i$, which lie on the opposite side of their common intersection of these balls. This way, the common intersection and this new input vector define the minimum enclosing ball. Additionally, we place an input vector in  $z\cdot(1,\ldots,1)$ where $y$ is chosen such that the midpoint of the largest ball of the construction is in $y\cdot(1,\ldots,1)$. The distance between the ball and the hyperplane above should be smaller than $\frac{1}{d}$. This part of the construction requires $d+1$ additional input vectors, resulting in $(d+1)t -d-1$ input vectors in total.

    Assume a setting with $(d+1)t -d-1$ vectors placed as above in $\mathbb{R}^d$. Observe that any two minimum enclosing balls must have at least $n-2t$ input vectors in their intersection. Further, by construction, any subset of $d$ balls has a common intersection, where $t-2$ input vectors are placed. Therefore, any two balls from the intersection have $d(t-2)= dt-2d$ input vectors in common. We need to make sure that $n-2t=(d-1)t-d-1 \ge dt-2d$. This inequality holds for $t\ge d-1$.

    The presented construction shows that the intersection of all minimum enclosing balls may be empty. According to Theorem~\ref{thm:safeAgreementNecessary}, this means that the \MEB validity condition cannot be satisfied.

\begin{figure}
\centering
\begin{subfigure}{.5\textwidth}
  \centering
  \includegraphics[width=0.65\linewidth]{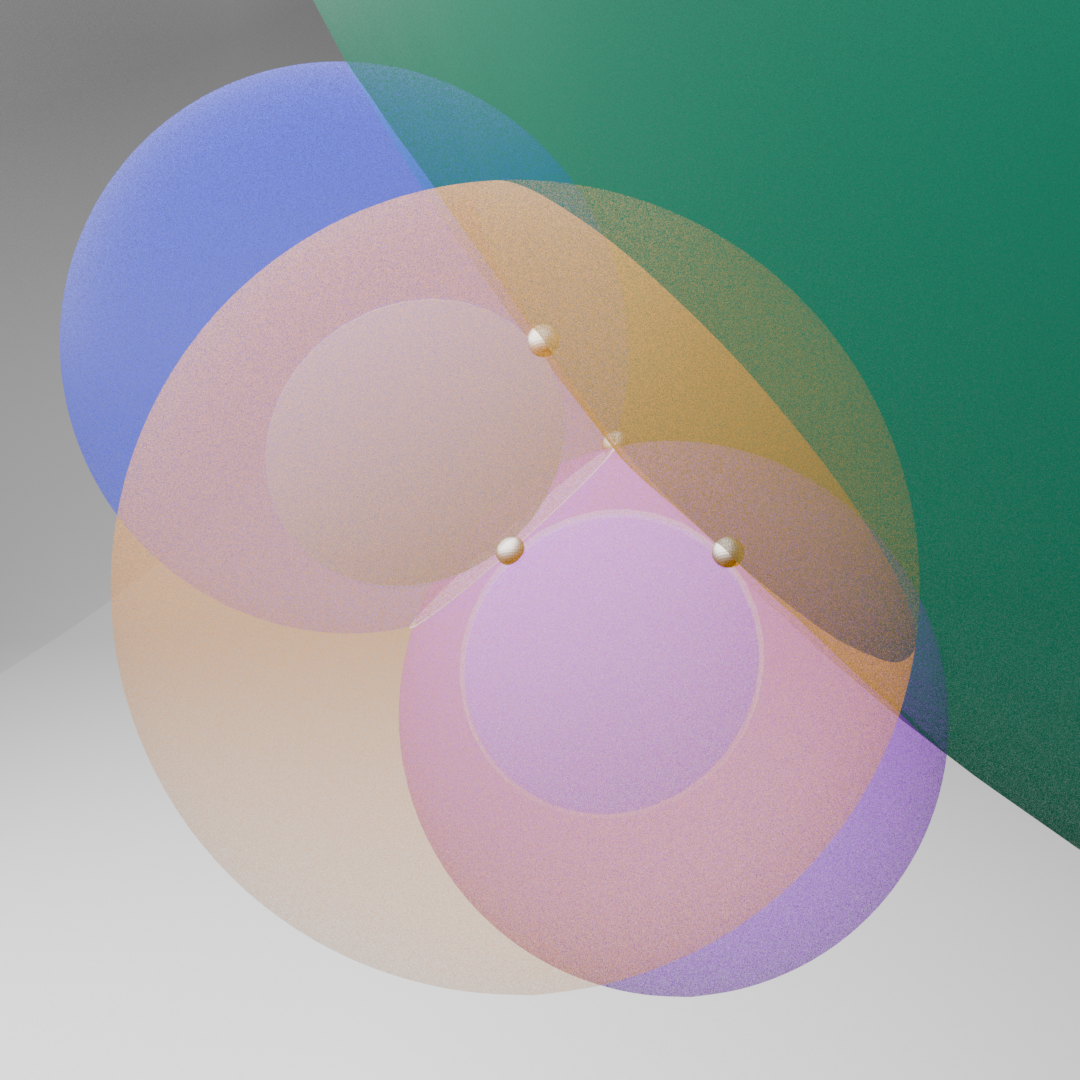}
  \caption{View from the side}
  \label{subfig:sideView}
\end{subfigure}%
\begin{subfigure}{.5\textwidth}
  \centering
  \includegraphics[width=0.65\linewidth]{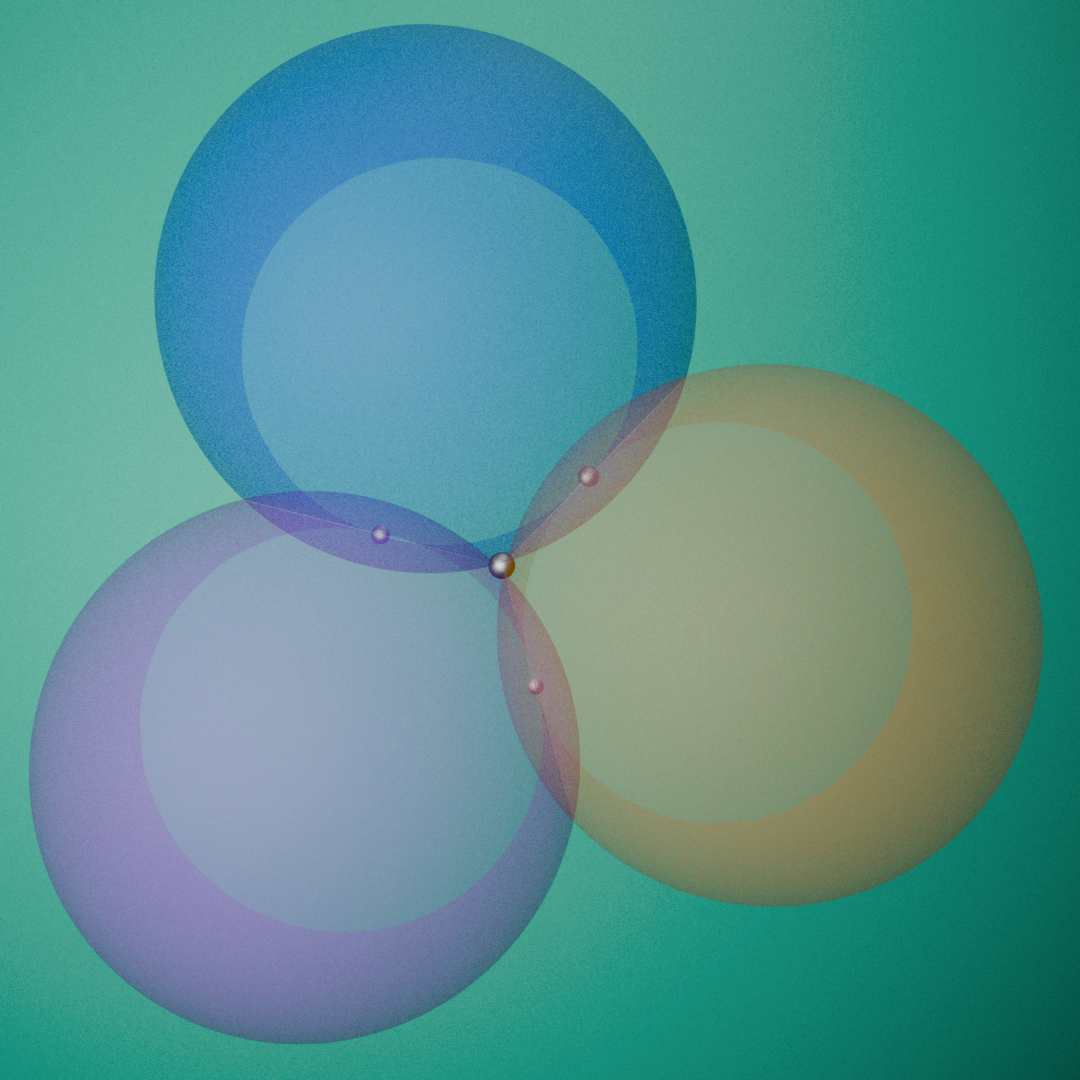}
  \caption{View from the origin}
  \label{subfig:bottomView}
\end{subfigure}
\caption{An example of the lower bound construction for $d=3$. The points in the figures represent the intersection points for each subset of three balls. Observe that there is a visible hole between the balls showing that there is no intersection between all four balls. Note that the input vectors for this construction are chosen as described in the proof of Theorem~\ref{thm:BallValidityLB}.}
\label{fig:LBconstruction}
\end{figure}
Figure~\ref{fig:LBconstruction} presents this construction in three dimensions. 
\end{proof}

\section{Optimal algorithm for \MEB validity}\label{sec:optimal}

In this section, we present positive results, by showing that agreement inside \safeMEB is sufficient to satisfy \MEB validity, and by presenting an optimal algorithm for $c$-\MEB validity, under the assumption $n>2t$.

\begin{lemma}
    $n>(d+1)t$ nodes are sufficient to satisfy \MEB validity.
\end{lemma}
\begin{proof}
    This follows trivially from the upper bound on convex validity~\cite{mendes2015multidimensional} which was proven using Helly's theorem~\cite{danzer1963helly}. In this bound, it was shown that, for $n>(d+1)t$, the convex hulls of all subsets of $n-t$ nodes have a non-empty intersection. Note that if the convex hulls of any subset of $n-t$ nodes intersect, so do the respective minimum enclosing balls.
\end{proof}

In the case when $n\leq(d+1)t$, \safeMEB does not necessarily exist, as the balls might have an empty intersection. Hence, exact \MEB validity is not satisfiable. Since it is not possible to know which \MEB consists of only honest nodes, we therefore consider a problem that minimizes the mistake of choosing the wrong ball. In particular, for a subset $T \subseteq P$ with $|T| = n-t$, let $\MEB(T)$ have center $c_T$ and radius $r_T$.
For a candidate output $y \in \mathbb{R}^d$, define $\phi_T(y)$ as:
\begin{align*}
    \phi_T(y) &= \frac{dist_2(y-\MEB(T))}{r_T}
    = \frac{\max\{0,\|y-c_T\|_2 - r_T\}}{r_T}
\end{align*}
Using this definition, we provide an optimal algorithm for $c$-\MEB validity shown in Algorithm \ref{alg:minmax-meb}.

\begin{algorithm}[b]
\caption{MinMax-MEB}
\label{alg:minmax-meb}
\begin{algorithmic}[1]
\State for every subset $T \subseteq P$ with $|T| = n-t$, compute $MEB(T)$ with center $c_T$ and radius $r_T$
\If{there exists $T$ with $r_T=0$}
\State \Return $c_T$
\Else 
\State choose $y^*\in \arg \min_{y\in \mathbb{R}^d} \max_{\substack{T \subseteq P \\ |T| = n-t}} \phi_T(y)$
\EndIf
\State \Return $y^*$
\end{algorithmic}
\end{algorithm}

\begin{theorem}
    Algorithm~\ref{alg:minmax-meb} solves the $c$-\MEB-validity optimally, if $n>2t$.
\end{theorem}
\begin{proof}
    %This somehow comes from the statement that any of the balls can be the true one (see corollary), and then the optimization problem is exactly the way to minimize the mistake. 
    Let $\mathcal{T}=\{T\subseteq P \mid |T|=n-t\}$. For every $T\in\mathcal{T}$, let $\MEB(T)$ have center $c_T$ and radius $r_T$. 
    We first consider the case where there exists $T_0\in\mathcal{T}$ with $r_{T_0}=0$ and center $c_{T_0}$. Since $n>2t$, every other subset $T\in\mathcal{T}$ intersects $T_0$. Hence every such $T$ contains the point $c_{T_0}$, and therefore $c_{T_0}\in\MEB(T)$ for all $T\in\mathcal{T}$. Thus $c_{T_0}$ lies in the exact safe \MEB area, and returning $c_{T_0}$ satisfies exact \MEB-validity, i.e., $c=1$.

    It remains to consider the case, in which all MEBs have a non-zero radius. Since it is not known which of the \MEB is defined on the subset containing honest nodes only, each subset might be the correct one. Algorithm~\ref{alg:minmax-meb} therefore computes the point, which minimizes the distance to the \MEB w.r.t. the ball radius, computed on all possible subsets of nodes of size $n-t$. This implies that for each possible \MEB, the Algorithm~\ref{alg:minmax-meb} satisfies $c$-\MEB validity. 
\end{proof}

Next, we present an upper bound on the relaxation $c$ used by the optimal algorithm to solve $c$-\MEB validity. This upper bound is based on the Descartes' theorem on the kissing circles. We use the statement of the theorem based on bends that originates in Soddy's poem ``The Kiss Precise''~\cite{soddy1936kiss}, or rather on the fourth verse of the poem (its generalization to $n$ dimensions), presented by Gosset~\cite{soddy1937kiss}. 

\begin{theorem}[Soddy-Gosset Theorem~\cite{soddy1937kiss}]\label{thm:soddy-gosset}
    Given $d+2$ mutually tangent spheres in $\mathbb{R}^d$. Let $r_i, i\in [d+2]$ denote the radii of the spheres and $b_i = \frac{1}{r_i}, i\in [d+2]$ their bends. Then,
    $$\left(\sum\limits_{i=1}^{d+2} b_i\right)^2 = d\cdot \sum\limits_{i=1}^{d+2} b_i^2.$$
\end{theorem}

Formal proofs of this theorem can be found in~\cite{SoddyGossetFirstProofTinyBitDanish, SoddyGossetOtherProofEnglish}.

\begin{theorem}\label{thm:upperBoundOpt}
    For $n>2t$, the optimal $c$ for $c$-\MEB-validity can be upper bounded by $1+\frac{k-1}{k+1+\sqrt{2(k+1)k}} < \sqrt{2}$, where $k=\min\{d, \binom{n}{n-t}-1\}$.
\end{theorem}
\begin{proof}
We consider a subset of $k+1$ minimum enclosing balls with an empty intersection. If the intersection of any subset of $k+1$ balls were not empty, by Helly's theorem~\cite{danzer1963helly}, the intersection of all minimum enclosing balls would be non-empty too. Thus, the optimal algorithm would choose a point inside the intersection of all minimum enclosing balls. Note that if there are more than $k+1$ balls in the construction that do not intersect, we can consider the subset with the worst distance ratios to the respective balls. 
We further assume that the midpoints of the $k+1$ balls are in a general position in a $k$-dimensional subspace of $\mathbb{R}^d$. If the midpoints lied in a lower dimensional space, the optimal solution would also lie in this subspace, and thus we could consider the following construction with fewer balls in a smaller subspace. 
In the following, we will make use of the property that any pair of minimum enclosing balls must intersect. This property holds for $n>2t$. Instead of considering proper intersections, we will construct a tangent ball, which allows us to apply the Soddy-Gosset theorem.
Note that for $k\ge 3$, this assumption is too weak, as by construction, minimum enclosing balls of any subset of $d$ balls must intersect as well. Thus, the following upper bound on $c$ is tight for the case $k=2$, but the actual bound is likely smaller for larger $k$.

Consider now a set of $k+1$ tangent balls, i.e., a setting where each pair of balls intersects in exactly one point. Note that there is a $(k+2)$-nd ball that can be inserted in the hole between the $k+1$ tangent balls such that it is tangent to each of the balls, denoted the inner Soddy ball. We will now consider the relaxation factor needed to include the midpoint of the inner Soddy ball in the intersection of all scaled balls. Observe that, if all minimum enclosing balls have the same radius, midpoint of the inner Soddy ball corresponds to the optimal intersection point. 

Let $r_1,\ldots,r_{k+1}$ denote the radii of the considered minimum enclosing balls, and the $b_1,\ldots,b_{k+1}$ their bend, that is, $b_i=\frac{1}{r_i}\ \forall i \in [k+1]$. Let $r_{k+2}$ and $b_{k+2}$ denote the radius and the bend of the inner Soddy ball. By the Soddy-Gosset theorem (Theorem~\ref{thm:soddy-gosset}), the bend of the two Soddy balls can be computed as 

\begin{align}
    b_{k+1} = \frac{\sum\limits_{i=1}^{k+1}b_i \pm \sqrt{k\cdot\left(\left(\sum\limits_{i=1}^{k+1} b_i\right)^2 -(k-1)\cdot\sum\limits_{i=1}^{k+1} b_i^2\right)}}{k}.\label{eq:innerSoddyBall}
\end{align}
Note that a positive bend corresponds to the inner circle in this formula. 
Assume now that the largest minimum enclosing ball has radius $1$. Observe that the radius of the inner Soddy ball is minimized if all minimum enclosing balls have the same radius $1$. This can be verified by starting with balls of radius $1$ and decreasing the radius of one of the balls, or increasing the bend. Note that Equation~\ref{eq:innerSoddyBall} also increases monotonically in this case. For $k=2$, this also shows that the radius is maximized for the case where the midpoint of the inner Soddy ball corresponds to the optimal intersection point chosen by an optimal algorithm solving $c$-\MEB validity. 

If all radii are $1$, the formula for the bend of the inner Soddy ball reduces to $$r_{k+2} = \frac{1}{b_{k+2}} = \frac{k-1}{k+1+\sqrt{2(k+1)k}}.$$

The relaxation factor $c$ can be computed as $1+\frac{k-1}{k+1+\sqrt{2(k+1)k}}$ which can be upper bounded by $\sqrt{2}$ for $k\rightarrow\infty$.
\end{proof}

Once the candidate MEBs are fixed, Algorithm~\ref{alg:minmax-meb} can be viewed as a second-order cone program, implying that the optimization problem is convex. However, the computational bottleneck is finding all candidate balls. Next, we study standard aggregation rules with explicit relaxed-\MEB guarantees.

%\todo[inline]{If we assume that there can be intersections, then the worst case should be when all balls have the same size, and any $d$ intersect in exactly one point (which they have to). For this case, we should be able to compute the upper bound on $c$, as it is a special case (example as above, but the balls are shifted toward the midpoint of the inner Soddy's circle until any three intersect). It would be better to have a general formula, but we somehow need to start with a configuration where all balls just touch, and then increase the radii until they intersect...}
%\todo[inline]{We somehow need to mention that the problem is either easy or hard to solve. I assume it is hard since we sometimes need to compute the geometric median. But I am not perfectly sure}

% \begin{lemma}
%     there is a lower bound showing that the above approximation $x$ is tight, hopefully already in $2$ dimensions, otherwise we are a bit screwed...
% \end{lemma}

\section{Existing algorithms satisfying $c$-\MEB-validity}\label{sec:existing_algorithms}
In the following, we examine the use of known algorithms, such as minimum-diameter averaging (MDA), geometric median and medoid, for \MEB validity. Our objective is to provide guarantees on the output of these algorithms w.r.t. the minimum enclosing ball of the honest input. For each algorithm, we provide a constant $c$, for which the algorithm satisfies $c$-\MEB validity. First, we define MDA, geometric median and medoid. 

\begin{definition}[MDA]
    Let $P=\{v_1,\dots,v_n\}\subset \mathbb{R}^d$ and let $t$ be the maximal number of Byzantine nodes. Define the diameter of a set $M\subseteq P$ as $D_M = \max_{v_i,v_j\in P} \|v_i-v_j\|_2$.
An MDA output is obtained as follows: choose a subset
$M \in \arg\min_{\substack{M\subseteq P \\ |M|=n-t}} D_M$, and output $ MDA = \frac{1}{|M|} \sum_{v_i\in M} v_i$.
\end{definition}

\begin{definition}[Geometric Median]
    Consider a set of $n$ nodes $\{v_{1},v_{2},\dots ,v_{n}\}$ with each $v_{i}\in \mathbb {R} ^{d}$, the geometric median of this set, denoted $\mu$, is defined as
    ${\underset {\mu\in \mathbb {R} ^{d}}{\operatorname {arg\,min} }}\sum _{i=1}^{n}\left\|v_{i}-\mu \right\|_{2}.$
\end{definition}

\begin{definition}[Medoid]
    Let $P=\{v_1,\dots,v_n\}\subset \mathbb{R}^d$. A medoid of $P$ is any input point
$y \in \arg\min_{v_i\in P} \sum_{j=1}^n \|v_i - v_j\|_2$.
\end{definition}

We begin with MDA, which selects a minimum-diameter subset and averages its elements and show that it satisfies $\left(1+\frac{2t}{n-t}\right)$-\MEB validity. 

\begin{theorem}
    MDA satisfies $\left(1+\frac{2t}{n-t}\right)$-\MEB validity. 
\end{theorem}
\begin{proof}
    Let $r^*$ denote the radius of the \MEB around honest inputs $H$, with $|H| \geq n-t$ and let $M$ be the subset of $n-t$ nodes which have the smallest diameter $D_M$. The diameter of subset $M$ is at most twice the radius of the true \MEB, i.e. $D_M \leq 2r^*$. 
    Now, the distance between MDA and the center of the ball is 
    \begin{align} \label{eq:mda}
    \|MDA-C^*\|_2 \leq \frac{1}{n-t}\sum_{x\in M} \|x-C^*\|_2.
    \end{align} The subset $M$ consists of exactly $n-t$ nodes, implying that there is at least one honest node contained in $M$. Therefore, we can split the subset $M$ into honest $H$ and Byzantine $B$ subset of nodes. For each $h\in H$, the distance to the center of the \MEB is bounded by $\|h-C^*\|_2 \leq r^*$. For each $b\in B$, using triangle inequality, we show that $b$ is at distance at most $3r^*$ from the center of the \MEB:
     $\|b - C_{\MEB}^*\|_2< \| b-u \|_2 + \|h-C_{\MEB}^*\|_2$, where $u$ is a honest node contained in $M$. 
     Since $u\in \MEB$, then $\|u-C_{\MEB}^*\|_2=r^*$, whereas ${u,v}\in M$ implies $\| b-u \|_2 \leq 2r^*$. Together it leads to $\|b - C_{\MEB}^*\|_2 \leq 3r^*$. 
     Finally, we can plug these results into Equation \ref{eq:mda}. 
     \begin{align*}
         \|MDA-C^*\|_2 &\leq \frac{|H|\cdot r^*+|B|\cdot 3r^*}{n-t} \leq \frac{(n-t-|B|)\cdot r^*+|B|\cdot 3r^*}{n-t} \\
       &\leq \left( 1+\frac{2|B|}{n-t}\right)r^* \leq \left(1+\frac{2t}{n-t}\right)r^*
     \end{align*}
\end{proof}

Next, we examine the medoid. Unlike MDA, the medoid minimizes the total distance to all inputs and restricts the output to be one of the input nodes. 

%In the following, we present a summary of our results that are formally stated and proved in the appendix:

We show that the medoid satisfies $\frac{3n-2t}{n-2t}$-\MEB validity and prove that it cannot satisfy $c$-\MEB validity, for $c\leq2$.

\begin{theorem}
    Medoid satisfies $\frac{3n-2t}{n-2t}$-\MEB validity.
\end{theorem}
\begin{proof}
    Let $P$ be the set of points, consisting of honest nodes in subset $H$ and Byzantine nodes in subset $B$. Let the true \MEB have the center $C^*$ with radius $r^*$. Denote the medoid by $m$ and the distance between the medoid and the center of \MEB by $D= \|m-C^*\|_2$. For each honest node $h\in H$, the following is true per definition: $\sum_{p\in P}\|m-p\|_2 \leq \sum_{p\in P}\|h-p\|_2$. We can split the set $P$ into honest and Byzantine nodes:
    
    \begin{equation}\label{eq:split}
        \sum_{u\in H}\|m-u\|_2 + \sum_{b\in B}\|m-b\|_2 \leq \sum_{u\in H}\|h-u\|_2 + \sum_{b\in B}\|h-b\|_2 \\
    \end{equation}
    For any honest  $u\in H$
    \begin{align*}
       \|m-u\|_2 &\geq \|m-C^*\|_2-\|u-C^*\|_2 \geq D-r^* 
    \end{align*}
    holds due to the triangle inequality.
    \begin{align*}
        \sum_{u\in H}\|m-u\|_2 &\geq | H| \cdot (D-r^*) \quad\forall u \in H.
    \end{align*}
    Now, for the right part of the Equation \ref{eq:split}, we bound $\sum_{u\in H}\|h-u\|_2 \leq | H| \cdot 2r^*$, as two honest nodes are at most $2r^*$ apart. Then, for any $b\in B$ the following triangle inequality applies:
    \begin{align*}
        \|h-b\|_2 &\leq \|h-m\|_2 + \|m-b\|_2 \\
        \sum_{b\in B}\|h-b\|_2 &\leq t\cdot \|h-m\|_2 + \sum_{b\in B}\|m-b\|_2 \quad \forall b \in B.\notag
    \end{align*}
    Now we plug this into Equation \ref{eq:split}:
    \begin{align*}
        | H| \cdot (D-r^*) +\sum_{b\in B}\|m-b\|_2 &\leq | H| \cdot 2r^* + t\cdot \|h-m\|_2 + \sum_{b\in B}\|m-b\|_2 \\
        | H| \cdot (D-r^*) &\leq | H| \cdot 2r^* + t\cdot (\|m-C^*\|_2+\|h-C^*\|_2) \\
        | H| \cdot (D-r^*) &\leq | H| \cdot 2r^* + t\cdot (D+r^*) \\
        (|H| -t)\cdot D &\leq r^* (3| H| +t)\\
    \end{align*}
    In the worst case if $|H| =n-t$, then: $D\leq \frac{3n-2t}{n-2t}r^*$, meaning the distance between the medoid and the center of the \MEB is at most $\frac{3n-2t}{n-2t}r^*$.
\end{proof}

\begin{theorem}\label{lem:medoid-impossibility}
    Medoid cannot satisfy $c$-\MEB validity for $c\leq 2$.
\end{theorem}
\begin{proof}
    Consider the following example with $n=2t+1$ nodes divided into subsets $A, B, C, D, E$ and $F$. Place subsets $A$ and $D$ of size $\frac{t}{2}-1$ in $(0,0)$ and $(4,0)$, respectively. Then place subsets $B$ and $C$ of size $\frac{t}{2}$ in $(1,1)$ and $(3,1)$. Finally, subset $E$ which consists of 2 nodes is found in $(2,0)$ and subset $F$ with one node in $(2,x)$, where $x$ is an arbitrary coordinate. The medoid of these input nodes is a node from subset $B$ or $C$, as both subsets consist of nodes which minimize the sum of distance to all other nodes and the construction is symmetric. 
    Note that the minimum enclosing ball is defined with at least $t+1$ nodes. Without knowing where the Byzantine parties are at, we consider three extreme cases for minimum enclosing ball. The first possible \MEB could contain all nodes in subsets $A,B$ and $E$, with radius $r=1$. The second possible \MEB could consist of subsets $E, C$ and $D$, with radius $r=1$. Finally the third possible \MEB can be computed on all nodes from subset $B, C$ and $F$, with radius $r=\frac{x}{2}$. Now, if the medoid is a node from subset $B$, the true \MEB could be the one consisting of nodes from subsets $C,E$ and $D$ and hence the medoid is outside the true \MEB. Similarly, if the medoid is a node from subset $C$, then the true \MEB could be around $A,B$ and $E$. Hence, the medoid is outside the \MEB. In both cases, the distance from the medoid to the true \MEB is $\sqrt{5}$, which implies $\sqrt{5} \le c\cdot 1$, for all $c\leq2$. Therefore, medoid cannot satisfy $c$-\MEB for $c\leq2$.
    %\todo[inline]{add a figure}
\end{proof}

% \begin{theorem}
%     Medoid satisfies $\frac{3n-2t}{n-2t}$-\MEB validity.
% \end{theorem}
% \begin{proof}
%     See Theorem~\ref{thm:medoid-meb-appendix}.
% \end{proof}

% \begin{lemma}
%     MDA satisfies $3$-\MEB validity for $n>3t$.
% \end{lemma}
% \begin{proof}
%     Let $r^*$ denote the radius of the \MEB around honest inputs $C$, with $| C| \geq n-t$ and let $M$ be the subset of $n-t$ nodes which have the smallest diameter $D_M$. The diameter of subset $M$ is at most twice the radius of the true \MEB, i.e. $D_M \leq 2r^*$. Since $M$ contains exactly $n-t$ nodes, and $n>3t$, there must be at least one honest node $v$ contained in $M$. Per  definition, $v$ is inside the true \MEB. 
%     Using triangle inequality, we show that each node in $u\in M$ is at distance at most $3r^*$ from the center of the \MEB:
%     $\|u - C_{\MEB}^*\|< \| u-v \| + \|v-C_{\MEB}^*\|$. 
%     Since $v\in \MEB$, then $\|v-C_{\MEB}^*\|=r^*$, whereas ${u,v}\in M$ implies $\| u-v \| \leq 2r^*$. Together it leads to $\|u - C_{\MEB}^*\| \leq 3r^*$. Therefore, $M \in 3-\MEB$. Finally, the $3-\MEB$ is convex, so the average of the nodes in $M$ is also inside $3-\MEB$. 
% \end{proof}

% \begin{theorem}\label{lem:medoid}
%     Medoid cannot satisfy $c$-\MEB validity for $c\leq 2$.
% \end{theorem}
% \begin{proof}
%     See Theorem~\ref{lem:medoid-impossibility-appendix}.
% \end{proof}

The geometric median, unlike the medoid, is not restricted to outputting input points and instead minimizes the sum of total Euclidean distances to all nodes over $\mathbb{R}^d$. 
%Also here, we present a summary of the results that are further discussed in the appendix:

% \begin{itemize}
%     \item Geometric median satisfies $\frac{2(n-t)}{n-2t}$-\MEB validity (Theorem~\ref{thm:geom-meb-appendix}),
%     \item Geometric median does not satisfy $1.08$-\MEB validity (Theorem~\ref{lem:geom-impossibility-appendix}),
%     \item Geometric median does not satisfy convex validity (Theorem~\ref{thm:geom-convex-appendix}).
% \end{itemize}

\begin{theorem}
    Geometric median satisfies $\frac{2(n-t)}{n-2t}$-\MEB validity.
\end{theorem}
\begin{proof}
    The proof is a modification of Lemma 24 \cite{10.1145/2897518.2897647} and Lemma C.6 \cite{10.5555/3737916.3739386}.
    Let $P$ be the set of points, consisting of honest nodes in subset $H$ and Byzantine nodes in subset $B$. Let the true \MEB have the center $C^*$ with radius $r^*$. Denote the geometric median by $\mu$ and the distance between the geometric median and the center of \MEB by $D= \|\mu-C^*\|_2$. The following is true per definition of the geometric median: $ \sum_{p \in P}\|\mu-p\| \leq  \sum_{p\in P}\|C^*-p\|_2$. We can split the set $P$ into honest and Byzantine nodes:
    \begin{equation} \label{eq:geom-split}
     \sum_{h\in H}\|\mu-h\|_2+ \sum_{b \in B}\|\mu-b\|_2 \leq  \sum_{h \in H}\|C^*-h\|_2+\sum_{b \in B}\|C^*-b\|_2.
    \end{equation}
    For any honest $h \in H$: $\|\mu-h\|_2 \geq \|\mu-C^*\|_2 - \|h-C^*\|_2 \geq D-r^*$.
    
    \noindent For all honest nodes $h\in H$: $\sum_{h\in H}\|\mu-h\|_2 \geq |H| \cdot (D-r^*)$.
    
    Now, for the right part of the Equation \ref{eq:geom-split}, we bound $\sum_{h\in H}\|C^*-h\|_2 \leq |H|\cdot r^*$, as any honest node is at most $r^*$ away from the center of the \MEB. Then, for any $b\in B$ the following triangle inequality applies:
    \begin{align*}
        \|C^*-b\|_2 &\leq \|C^*-\mu\|_2 + \|\mu-b\|_2 \\
        \sum_{b \in B}\|C^*-b\|_2 &\leq t\cdot D+ \sum_{b \in B}\|\mu-b\|_2 \quad \forall  b \in B.
    \end{align*}
    Now, we plug this into Equation \ref{eq:geom-split}:
    \begin{align*}
        |H| \cdot (D-r^*) + \sum_{b \in B}\|\mu-b\|_2 &\leq |H| \cdot r^* + t\cdot D +\sum_{b \in B}\|\mu-b\|_2 \\
        D(|H|-t) &\leq 2\cdot |H|\cdot r^*. 
    \end{align*}
    In the worst case if $|H|=n-t$, then $D\leq \frac{2(n-t)}{n-2t}r^*$. This implies that the distance between the geometric median and the center of the \MEB is at most $\frac{2(n-t)}{n-2t}r^*$
\end{proof}

\begin{theorem}
    Geometric median does not satisfy $1.08$-\MEB validity.
\end{theorem}
\begin{proof}
    \begin{proof}
    Consider an example similar to the example from Lemma \ref{lem:medoid-impossibility} with $n=3t-3$ nodes divided into subsets $A, B, C, D, E$ and $F$. Subsets $A$ and $D$ contain one node each and are placed in in $(0,0)$ and $(4,0)$, respectively. Then place subsets $B$ and $C$ of size $t-2$ in $(1,1)$ and $(3,1)$. Finally, subset $E$ which consists of $t-2$ nodes is found in $(2,0)$ and subset $F$ with one node in $(2,x)$, where $x$ is an arbitrary coordinate. 
    We consider three extreme cases for the minimum enclosing ball defined on a subset of size $2t-3$. One possible \MEB contains all nodes in subsets $A$, $B$, and $E$, with radius $r = 1$. A second possible \MEB contains the nodes in subsets $E$, $C$, and $D$, also with radius $r = 1$. A third possible MEB is formed from all nodes in subsets $B$, $C$, and $F$, with radius $r = \frac{x}{2}$.
    Since the construction is symmetric and points in subsets $B, C$ and $E$ dominate over the subsets containing one point, the geometric median is the Fermat's point of the triangle defined by the nodes in subsets $B, C$ and $E$, found in $(2,1-\frac{1}{\sqrt{3}})$.  Note that the geometric median is not contained in any of the three extreme \MEB. The distance between the center of the two \MEB with $r=1$ and the geometric median is $\sqrt{\frac{7}{3}-\frac{2}{\sqrt{3}}}\approx 1.0857$. 
\end{proof}
\end{proof}

\begin{theorem}
    Geometric median does not satisfy convex validity.
\end{theorem}
\begin{proof}
    \begin{proof}
Consider the following example with $n=3t+1$ input points. Place $t+1$ points into the origin $(0,0)$, then place $t$ points in $(1,0)$ and the other $t$ points in $(0,1)$. The geometric median of these input points converges to the Fermat's point of the triangle, when $t$ is large. However, the convex hull containing honest nodes could be defined on the line between the origin and one of the sides. Hence, the geometric median is not contained in the convex hull of the honest nodes. 
%\todo[inline]{add a figure, actually probably no figures in this section}
\end{proof}
\end{proof}
%Finally, in Section~\ref{sec:relations} we present the relation of \MEB and $c$-\MEB validity to existing validity conditions for vector aggregation. 

\section{Relations between validity conditions}\label{sec:relations}

To put \MEB validity into perspective with other validity conditions, we make comparisons between \MEB validity and existing validity conditions in vector aggregation. We consider convex validity, $(\delta, p)$-relaxed convex validity and box validity. Since these conditions define different valid regions, it is natural to ask how they relate to one another. In particular, we focus on whether one validity condition implies another, and how they compare quantitatively. We first define considered validity conditions.
\smallsep
%\begin{description}
    \noindent\textbf{Convex Validity:} The agreement vector has to be inside the convex hull of $C$, denoted by $\convexhull(C)$. 
%\end{description}
\smallsep
It has been shown that, to agree on a vector inside the convex hull of $H$ under Byzantine failures, one has to agree inside a so-called \emph{safe area}. This is the intersection of all convex hulls computed on subsets of $n-t$ vectors from all $n$ received vectors. This condition is also the reason why convex validity has a poor resilience of $n>(d+1)t$. For $n\le(d+1)t$, there is no guarantee that the convex hulls have a common intersection.

The second validity condition that we consider in this paper is the box validity: 
\smallsep
%\begin{description}
    \noindent\textbf{Box Validity:} The agreement vector has to be inside the smallest coordinate-parallel hyperrectangle containing $H$, named trusted box \TB. 
%\end{description}
\smallsep
Box validity can be viewed as convex validity applied separately in each dimension. Observe that both the box and the convex validity conditions restrict the agreement vector to only take values in dimensions that are spanned by the convex hull or the hyperrectangle. 
\smallsep
%\begin{description}
    \noindent\textbf{$\bm{(\delta,p)}$-relaxed convex validity:} The agreement vector has to be inside the set $\{u\big\vert \Vert u - v \Vert_p \le \delta\ \ \forall\ v\in \convexhull(H) \}$.
%\end{description}
\smallsep
Where $\convexhull(H)$ denotes the convex hull of the set $H$, and $0<p\le\infty$.

We start with convex validity and show that convex validity implies \MEB validity.

\begin{lemma}
    Convex validity implies \MEB validity.
\end{lemma}
\begin{proof}
    The convex hull \convexhull and minimum enclosing ball \MEB are defined on the same set of honest nodes $H$. Hence, $\convexhull \subseteq \MEB$.
\end{proof}

Next, we show that the reversed direction does not hold, i.e. \MEB does not necessarily satisfy convex validity.

\begin{lemma}
    \MEB validity does not imply convex validity.
\end{lemma}
\begin{proof}
    Consider the following configuration: place $\frac{|H|}{2}$ of the honest nodes in (0,0) and the remaining $\frac{|H|}{2}$ in (0,1). The convex hull \convexhull is the line segment connecting the two points. In contrast, \MEB is a $2d$ ball centered in the midpoint of the line segment, with radius $r^*=\frac{1}{2}$. The \convexhull is collapsed to the line segment in one dimension. Hence, $\MEB\nsubseteq \convexhull$, and thus $\MEB$ validity does not imply convex validity.
\end{proof}

%This should be similar to the box, since it is a subdimensional thing. And we should make clear that the fact that the ball is fully-dimensional makes our relaxation possible.

% \begin{lemma}
%      $c$-relaxed \MEB validity  satisfies $(2c\cdot diam,2)$-relaxed convex validity , where $diam(C)$ denotes the largest smallest distance between any pair of points in $C$.
% \end{lemma}

We now turn our focus to the $(\delta, p)$-relaxed convex validity condition, since it requires that the output of the algorithm is at distance at most $\delta$ from the convex hull. This relaxation increases the valid output area in all dimensions, unlike the convex hull which might be only a line segment.  

\begin{lemma}
     $c$-relaxed \MEB validity implies $((c+1)\cdot r^*,2)$-relaxed convex validity , where $r^*$ denotes the radius of the minimum enclosing ball \MEB.
\end{lemma}
\begin{proof}
    Let $y$ be an output point inside the $c$-\MEB, which has center $C^*$ and radius $r^*$. The distance between $y$ and the convex hull of honest nodes \convexhull is bounded by $dist(y,\convexhull) \leq \| y-x\|_2$ for each $x \in \convexhull$. Using triangle inequality, we can derive the following:
    \begin{align*}
        \|y-x\|_2 &\leq \|y-C^*\|_2 + \|C^*-x\|_2 \\
        &\leq c\cdot r^* + r^* \leq (c+1)r^*.
    \end{align*}
    Hence, $c$-\MEB satisfies $((c+1)\cdot r^*,2)$-relaxed convex validity. 
\end{proof}

\begin{lemma}
    $(\delta,2)$-relaxed convex validity implies $\left(1+\frac{2\delta}{diam(H)}\right)$-relaxed \MEB validity in general, where $diam(H)$ denotes the largest smallest distance between any pair of points in $H$.
\end{lemma}
\begin{proof}
    Let $y$ be a node inside the $(\delta,2)$-relaxed convex hull of the honest nodes. Then, there exists a point $x\in \convexhull$ such that $\|y-x\|_2 \leq \delta$. Further, $\|x-C^*\|_2 \leq r^*$, as $x\in \convexhull \subseteq \MEB$. Now, we can bound the distance between $y$ and the center of the \MEB:
    \begin{align*}
        \|y-C^*\|_2 &\leq \|C^*-x\|_2 + \|y-x\|_2 \\
        &\leq \delta + r^*.
    \end{align*}
    Now, $\delta + r^* = \left(1+ \frac{\delta}{r^*} \right)r^*  \leq \left(1+ \frac{\delta}{\frac{diam(H)}{2}} \right) r^*$. Hence, $(\delta,2)$-relaxed convex validity satisfies $\left(1+\frac{2\delta}{diam(H)}\right)$-relaxed \MEB validity.
    Note that if all honest nodes are in one point, then $diam(H)=0$ and \convexhull and the \MEB lie in the same point. 
\end{proof}

%I think we can compare it to the radius of the minimum enclosing ball, but it may be cleaner to use the diameter instead.\todo{fix}
Finally, we examine box validity and its relation to the \MEB validity. We show that the output of the algorithm satisfying box validity can be at most $\sqrt{d}\cdot r^*$ distance from the \MEB. We also show that satisfying \MEB does not provide any guarantees on box validity. 

\begin{lemma}
    Box validity implies $\sqrt{d}$-\MEB validity.
\end{lemma}
\begin{proof}
    The \TB and \MEB are defined on the same set of honest nodes $H$. This implies that for each honest node $h$: $\|h-C^*\|_2 \leq r^* $. For each coordinate $i$: $|h_i-C_i^*|\leq r^*$. Since the \TB is defined with the smallest and largest value in each dimension, its interval is $[C^*_i - r^*, C^*_i +r^*]$. Now, for each point $y \in TB$ the distance to the center of the ball is bounded by $|y_i - C^*_i| \leq r^*$ for each coordinate $i$. Finally, using the Euclidean distance, we can derive the following:
    \begin{align*}
        \| y-C^*\|_2^2 &= \sum_{i=1}^d (y_i - C^*_i)^2 
        \leq \sum_{i=1}^d (r^*)^2 
        \leq d(r^*)^2.
    \end{align*}
    Hence, each point $y\in TB$ is at most $\sqrt{d}\cdot r^*$ distance away from the \MEB. 
\end{proof}

\begin{lemma}
    \MEB validity does not imply box validity in general.
\end{lemma}
\begin{proof}
    Consider the following scenario: place $\frac{|H|}{2}$ of the honest nodes in (0,0) and the other $\frac{|H|}{2}$ in (0,1). The trusted box \TB is found on the line segment between the two points. However, \MEB is the unit disc centered in the midpoint of the line segment. \MEB is fully dimensional, whereas the \TB is collapsed to the line segment. Hence, $\MEB\nsubseteq \TB$, and thus $\MEB$ validity does not imply box validity.

\end{proof}
Note that this result holds for any relaxation of the box validity if we define it analogously to the box validity.

\section{Conclusion}\label{sec:conclusion}

\noindent\textbf{Summary.} We introduced a new practical validity condition based on minimum enclosing ball of all non-faulty nodes. Our results show that exact \MEB validity still suffers from resilience limitations that depend on the geometry of the input. 
We therefore provided a natural relaxation of the validity condition, where the radius of the minimum enclosing ball is increased by a constant factor $c$. Further, we presented an optimal algorithm for $c$-\MEB validity and proved that its relaxation ratio can be upper bounded by $c< \sqrt{2}$, if $n>2t$. Beside the optimal algorithm, we analyzed existing algorithms such as MDA, geometric median and medoid, and showed that they satisfy $c$-\MEB validity, where $c$ is at most constant, supporting the fact that these aggregation methods perform well in practice. Finally, we provide a systematic comparison between \MEB and existing validity conditions. We believe that \MEB validity provides a useful foundation for designing and analyzing robust aggregation methods that are both theoretically sound and practically effective in FL.

\noindent\textbf{Limitations.} There are many interesting avenues for future research. 
One question concerns computational efficiency. The optimal MinMax-\MEB rule requires computing candidate minimum enclosing balls for all subsets of size $n-t$, so the exact implementation becomes expensive when $t$ is large.
%, and is therefore combinatorial in $t$. 
Practical rules such as medoid, geometric median or MDA are more efficient, but achieve weaker relaxation factors. 
%This suggests a trade-off between computational cost and the strength of the validity guarantee.
Moreover, our model assumes a trusted server, reliable communication, Euclidean geometry and a known upper bound $t$ on the number of Byzantine clients. We view extending \MEB validity to fully decentralized FL as an important direction for future work.

Another limitation is that several of our relaxation bounds are not tight. For MinMax-\MEB, the bound $c<\sqrt{2}$ is tight in the two-dimensional case, but the optimal factor may be even better in higher dimensions. Similarly, for MDA, medoid and geometric median, it is not clear whether our bounds are tight. Closing these gaps is an interesting geometric question.
Additionally, one may ask how the relation is between multiplicative and additive relaxation of \MEB validity. One can define additive relaxation for \MEB validity similarly to~\cite{xiang_et_al:LIPIcs.OPODIS.2016.26}. Observe that the upper bound from Theorem~\ref{thm:upperBoundOpt} would also hold for the additive definition. The upper-bound construction from Theorem~\ref{thm:upperBoundOpt} also gives an additive guarantee, but the two relaxations may behave differently when the candidate minimum enclosing balls have very different radii. Understanding when multiplicative or additive relaxation is more appropriate for practical aggregation remains an interesting
direction for future work.

%This work also leaves a few geometric questions open, such as how to compute a tight upper bound on the relaxation factor $c$ of an optimal algorithm, or how to find lower bounds for known aggregation methods from the literature.

% maybe: The computational bottleneck of MinMax-MEB is the enumeration of all subsets of size n−t. However, the structure of minimum enclosing balls suggests that only a small number of extremal subsets are relevant in practice. In particular, MEBs are defined by at most d+1 support points, and many candidate balls are dominated by others with smaller radius. This enables several practical optimizations, including dominance pruning, randomized subset sampling, and heuristics based on extremal configurations. Exploring efficient approximations of MinMax-MEB remains an important direction for future work.

\bibliographystyle{unsrt}
\bibliography{literature}

%%%%%%%%%%%%%%%%%%%%%%%%%%%%%%%%%%%%%%%%%%%%%%%%%%%%%%%%%%%%
\newpage
\appendix

\end{document}